\documentclass{edm_article}
\usepackage{graphicx}
\usepackage{amssymb}
\usepackage{subcaption}
\usepackage[labelfont=bf,textfont=bf]{caption}
\usepackage[ruled,vlined,linesnumbered]{algorithm2e}
\newtheorem{theorem}{Theorem}
\usepackage{todonotes}

\newcommand{\R}{\mathbb{R}}
\newdef{definition}{Definition}

\begin{document}

\title{Fair-Capacitated Clustering}

\numberofauthors{4} 
 \author{
 \alignauthor Tai Le Quy \\
        \affaddr{Leibniz University Hannover}\\
        \affaddr{Hannover, Germany}\\
        \email{tai@l3s.de}
\alignauthor Arjun Roy\\
        \affaddr{Leibniz University Hannover}\\
        \affaddr{Hannover, Germany}\\
        \email{roy@l3s.de}
\alignauthor Gunnar Friege\\
        \affaddr{Leibniz University Hannover}\\
        \affaddr{Hannover, Germany}\\
        \email{friege@idmp.uni-hannover.de}
\and
\author \ Eirini Ntoutsi\\
        \affaddr{Free University of Berlin}\\
        \affaddr{Berlin, Germany}\\
        \email{eirini.ntoutsi@fu-berlin.de}
 }

\maketitle


\begin{abstract}
Traditionally, clustering algorithms focus on partitioning the data into groups of similar instances.
The similarity objective, however, is not sufficient in applications where a \emph{fair-representation} of the groups in terms of protected attributes like gender or race, is required for each cluster. 
Moreover, in many applications, to make the clusters useful for the end-user, a \emph{balanced cardinality}
among the clusters is required. Our motivation comes from the education domain where studies indicate that students might learn better in diverse student groups and of course groups of similar cardinality are more practical e.g., for group assignments.
To this end, we introduce the \emph{fair-capacitated clustering problem} that partitions the data into clusters of similar instances while ensuring cluster fairness 
and balancing cluster cardinalities. 
We propose a two-step solution to the problem: i) we rely on fairlets to generate minimal sets that satisfy the fair constraint and ii) we propose two approaches, namely hierarchical clustering and partitioning-based clustering, to obtain the fair-capacitated clustering. The hierarchical approach embeds the additional cardinality requirements during the merging step while the partitioning-based one alters the assignment step using a knapsack problem formulation to satisfy the additional requirements.
Our experiments on four educational datasets show that our approaches deliver well-balanced clusters in terms of both fairness and cardinality while maintaining a good clustering quality.
\end{abstract}

%


\section{Introduction}
\label{sec:introduction}
Machine learning (ML) plays a crucial role in decision-making in almost all areas of our lives, including areas of high societal impact, like healthcare and education. Our work's motivation comes from the education domain where ML-based decision-making has been used in a wide variety of tasks from student dropout prediction \cite{gardner2019evaluating}, forecasting on-time graduation of students \cite{hutt2019evaluating} to education admission decisions \cite{marcinkowski2020implications}. Recently, the issue of bias and discrimination in ML-based decision-making systems is receiving a lot of attention~\cite{ntoutsi2020bias} as there are many recorded incidents of discrimination (e.g., recidivism prediction~\cite{larson2016we}, grades prediction  \cite{bhopal2020impact,hubble2020level})
caused by such systems against individuals or groups or people on the basis of \emph{protected attributes} like gender, race etc. Bias in education is not a new problem, rather there is already a long literature on different sources of bias in education~\cite{meaney2019early} or students' data analysis \cite{bharara2018application} as well as studies on racial bias~\cite{warikoo2016examining} and 
gender bias~\cite{masterson2012empirical}. 
 However, ML-based decision-making systems have the potential to amplify prevalent biases or create new ones and therefore, fairness-aware ML approaches are required also for the educational domain.

In this work, we focus on fairness in clustering, i.e., the balance of members in cluster w.r.t protected attributes, as in educational activities, group assignments \cite{ford2003fair} and student team achievement divisions  \cite{tiantong2013student} are important tools in teaching and learning tasks to help students working together towards shared learning goals. 
Better communication, higher-order thinking, conflict management are several examples of the advantages of group assignments \cite{ford2003fair}. 
Clustering algorithms are effective solutions for partitioning students into groups of \emph{similar instances} \cite{bharara2018application,navarro2018comparison}. Traditional algorithms, however, focus solely on the similarity objective and do not consider the fairness of the resulting clusters w.r.t. protected attributes like gender or race. However, studies indicate that students might learn better in diverse student groups, e.g., mixed-gender groups~\cite{gnesdilow2013group,zhan2015effects}.
Lately, fair-clustering solutions have been proposed, e.g.,~\cite{chierichetti2017fair}, which aim to discover clusters with a fair representation regarding some protected attributes. 



In a teaching situation, one is often interested in certain group sizes which are usually between 2--4 students per group in primary, secondary and university education but might be much larger in adult education and Massive Open Online Courses (MOOCs). 
It is obvious that the size of the groups should be comparable to allow for a fair allocation of work among the students. Again, traditional clustering algorithms do not consider this requirement, and as a result, clusters of varying sizes might be extracted, reducing the usefulness and applicability of the partitioning for the end-user/teacher. This leads to the demand for clustering solutions that also take into account the size of the clusters. 
The problem is known as \emph{the capacitated clustering problem (CCP)} \cite{mulvey1984solving}, which aims to extract clusters with a limited capacity\footnote{We use the terms cluster capacity, cluster size and cluster cardinality interchangeably to refer to the number of instances in a cluster.} while minimizing the total dissimilarity in the clusters.
Capacitated clustering is useful in many applications, e.g., transferring goods/services from the service providers (post office, stores, etc.), garbage collection and salesforce territorial design \cite{negreiros2006capacitated}  to various customer locations \cite{geetha2009improved}. 
To the best of our knowledge, no solution exists that considers both fairness and capacity of clusters on top of the similarity objective.

To this end, we propose a new problem, the so-called \emph{fair-capacitated clustering} that ensures fairness and balanced cardinalities of the resulting clusters. We decompose the problem into two subproblems: i) the fairness-requirement compliance step that preserves fairness at a minimum threshold of balance score
and ii) the capacity-requirement compliance step that ensures clusters of comparable sizes. 
For the first step, we generate fairlets~\cite{chierichetti2017fair}, which are minimal sets that satisfy fair representation w.r.t. a protected attribute while approximately preserving the clustering objective. For the second step, we propose two solutions for two different clustering types, namely hierarchical and partitioning-based clustering, that consider the capacity constraint during the merge step (for the hierarchical approach) and  during the assignment step (for the partitioning approach). Experimental results, on four real datasets from the education domain, show that our methods result in fair and capacitated clusters while preserving the clustering quality.

The rest of our paper is structured as follows: Section \ref{sec:related} overviews the related work. The fair-capacitated clustering problem is introduced in Section \ref{sec:problem_definition}. Section \ref{sec:methods} describes our proposed approaches and section \ref{sec:experiment} presents the detail of experimental evaluation on educational datasets. Finally, the conclusion and outlook are summarized in Section \ref{sec:conclusion}.

\section{Related work}
\label{sec:related}
Chierichetti et al. \cite{chierichetti2017fair} introduced the fair clustering problem with the aim to ensure equal representation for each protected attribute, such as gender, in every cluster. In their formulation, each instance is assigned with one of two colors (red, blue). They proposed a two-phase approach: clustering all instances into fairlets - small clusters preserving the fairness measure, and then applying vanilla clustering methods ($k$-Center, $k$-Median) on those fairlets. 
Subsequent studies focus on generalization and scalability. Backurs et al. \cite{backurs2019scalable} presented an approximate fairlet decomposition algorithm which can formulate the fairlets in nearly linear time thus tackling the efficiency bottleneck of~\cite{chierichetti2017fair}.
Rösner and Schmidt \cite{rosner2018privacy} generalize the fair clustering problem to more than two protected attributes. A more generalized and tunable notion of fairness for clustering was introduced in Bera et al. \cite{bera2019fair}. 
Anshuman and Prasant \cite{chhabra2020fair} introduced a fair hierarchical agglomerative clustering method for multiple protected attributes.

The capacitated clustering problem (CCP), a combinatorial optimization problem, was first introduced by Mulvey and Beck~\cite{mulvey1984solving} who proposed solutions using heuristic and subgradient algorithms. 
Several approaches exist to improve the efficiency of solutions or CCP approaches for different cluster types. Khuller and Sussmann \cite{khuller2000capacitated}, for example, introduced an approximation algorithm 
for the capacitated $k$-Center problem. Geetha et al. \cite{geetha2009improved} improved $k$-Means algorithm for CCP by using a priority measure to assign points to their centroid.
Lam and Mittenthal \cite{lam2013capacitated} proposed a heuristic hierarchical clustering method for CCP to solve the multi-depot location-routing problem.


In this work, we introduce the problem of fair-capacitated clustering which builds upon notions from fair clustering and capacitated clustering. In particular, we build upon the notion of fairlets~\cite{chierichetti2017fair} to extract the minimal sets that preserve fairness. Regarding the CCP we follow the formulation of~\cite{mulvey1984solving} to ensure balanced cluster cardinalities. To the best of our knowledge, the combined problem has not been studied before and as already discussed, comprises a useful tool in many domains like education.

\section{Problem definition}
\label{sec:problem_definition}
Let $X \in \R^n$ be a set of instances to be clustered and let $d(): X \times X \rightarrow \R$ be the distance function.
For an integer $k$ we use $[k]$ to denote the set $\{1,2,...,k\}$.  
A \textit{k-clustering} $\mathcal{C}$ is a partition of $X$ into $k$ disjoint subsets, $\mathcal{C}=\{C_1, C_2,...,C_k\}$, called \textit{clusters} with $S=\{s_1, s_2,...,s_k\}$ be the corresponding cluster centers. 
The goal of clustering is to find an \textit{assignment}\footnote{We focus on hard clustering where each instance is 
only assigned to one cluster.} $\phi: X \rightarrow 
[k]$ 
that minimizes the  objective function:
\begin{equation}
\label{eq:eval}
    \mathcal{L}(X,\mathcal{C}) = \sum_{s_i \in S} \sum_{x \in C_i} d(x,s_i)
\end{equation}
As shown in Eq. \ref{eq:eval}, the goal is to find an assignment that minimizes  the sum of distances between each point $x\in X$and its corresponding cluster center $s_i \in S$. It is clear that such an assignment optimizes for similarity but does not consider fairness or capacity of the resulting clusters.

\textbf{Capacitated clustering:}
The goal of capacitated clustering \cite{mulvey1984solving} is to discover clusters of given capacities while still minimizing the distance objective  $\mathcal{L}(X,\mathcal{C})$. The capacity constraint is defined as an upper bound $Q_i$ on the cardinality of each cluster $C_i$:
\begin{equation}
    |C_i| \leq Q_i
\end{equation}
\textbf{Clustering fairness: }
We assume the existence of a binary protected attribute $P=\{0,1\}$, e.g.,  gender=\{``male", ``female"\}. Let $\psi: X \rightarrow P$ denotes the demographic group to which the point belongs, i.e., male or female. 

Fairness of a cluster is evaluated in terms of the balance score~\cite{chierichetti2017fair} as the minimum ratio between two groups.
\begin{equation}
\label{eq:cluster_balance}
    \resizebox{1.0\hsize}{!}{$balance(C_i)=\min\left(\frac{|\{x \in C_i \mid \psi(x)=0\}|}{|\{x \in C_i \mid \psi(x)=1\}|},
    \frac{|\{x \in C_i \mid \psi(x)=1\}|}{|\{x \in C_i \mid \psi(x)=0\}|}\right)$}
\end{equation}
Fairness of a clustering $\mathcal{C}$ equals to the balance of the least balanced cluster $C_i \in \mathcal{C}$.
\begin{equation}
\label{eq:clustering_balance}
    balance(\mathcal{C}) = \min_{C_i \in C} balance(C_i)
\end{equation}
We now introduce the problem of fair-capacitated clustering that combines all aforementioned objectives regarding distance, fairness and capacity.

\begin{definition} (Fair-capacitated clustering problem)\\
We define the problem of $(t,k,q)$-fair-capacitated clustering as finding a clustering $\mathcal{C}=\{C_1, \cdots C_k\}$ that partitions the data $X$ into $|\mathcal{C}|=k$ clusters such that the cardinality of each cluster $C_i \in \mathcal{C}$ does not exceed a threshold $q$, i.e., $|C_i|\leq q$ (\emph{the capacity constraint}), the balance of each cluster is at least  $t$, i.e.,  $balance(\mathcal{C})\geq t$ (\emph{the fairness constraint}), and minimizes \emph{the objective function} $\mathcal{L}(X,\mathcal{C})$.
Parameters $k, t, q$ are user defined  referring to the number of clusters, minimum balance threshold and maximum cluster capacity, respectively. 
\end{definition}

\section{Fair-capacitated clustering}
\label{sec:methods}
In this section, we propose two $(t,k,q)$ fair-capacitated clustering approaches, one for hierarchical clustering and the second for partitioning-based clustering.
For both cases, we decompose the complex problem into two simpler subproblems: i) the fairlet decomposition step that divides the original points into set of points, the so-called fairlets, each preserving a balance score subject to the balance threshold $t$ (Section~\ref{subsec:fairlets}) and ii) the final clustering step that clusters these fairlets into $k$ final clusters so that the cardinality constraint subject to the cardinality threshold $q$ is met. Step (ii) depends on the clustering type: for hierarchical clustering, the merge step needs to be changed (Section~\ref{subsec:hierarchical}), whereas for the partitioning-based clustering the assignment step needs to change (Section~\ref{subsec:kmedoid_knapsack}).

\subsection{Fairlet decomposition}
\label{subsec:fairlets}
Traditionally, the vanilla versions of clustering algorithms are not capable of ensuring fairness because they assign the data points to the closest center without the fairness consideration. Hence, if we could divide the original data set into subsets such that each of them satisfies the balance threshold $t$
then grouping these subsets to generate the final clustering would still preserve the fairness constraint. Each fair subset
is defined as a fairlet.
We follow the definition of fairlet decomposition by~\cite{chierichetti2017fair}.
\begin{definition}\label{def:fairlet} (Fairlet decomposition) \\Suppose that $balance(X) \geq t$ with $t = f/m$ for some integers $1\leq f\leq m$,
such that the greatest common divisor $gcd(f,m) = 1$. A decomposition $\mathcal{F} =\{F_1, F_2, ..., F_l$\} of $X$ is a fairlet decomposition if: i) each point $x \in X$ belongs to exactly one fairlet $F_j \in \mathcal{F}$, ii) $|F_j|\leq f+m$ for each $F_j \in \mathcal{F}$, i.e., the size of each fairlet is small, and iii) for each $F_j \in \mathcal{F}$, $balance(F_j) \geq t$, i.e., the balance of each fairlet satisfies the threshold $t$. Each $F_j$ is called a fairlet.
\end{definition}


By applying fairlet decomposition on the original dataset $X$, we obtain a set of fairlets $\mathcal{F} = \{F_1, F_2, \ldots, F_l\}$. For each fairlet $F_j$ we select randomly a point $r_j \in F_j$ as the \textit{center}. For a point $x \in X$, we denote $\gamma: X \rightarrow [1,l]$ as the index of the mapped fairlet. 

The second step, is to cluster the set of fairlets $\mathcal{F} = \{F_1, F_2, $ $ \ldots, F_l\}$ into $k$ final clusters, subject to the cardinality constraint. The clustering process is described below for the hierarchical clustering type (Section~\ref{subsec:hierarchical}) and for the partitioning-based clustering type (Section~\ref{subsec:kmedoid_knapsack}).
Clustering results in an assignment from fairless to final clusters: 
$\delta: \mathcal{F} \rightarrow [k]$. The final fair-capacitated clustering $\mathcal{C}$ can be determined by the overall assignment function
$\phi(x)= \delta(F_{\gamma(x)})$, where $\gamma(x)$ returns the index of the fairlet to which $x$ is mapped.

\subsection{Fair-capacitated  hierarchical clustering}
\label{subsec:hierarchical}
Given the set of fairlets: $\mathcal{F} = \{F_1, F_2, \ldots, F_l\}$, let $W = \{w_1, w_2, \ldots, w_l\}$ be their corresponding weights, where the weight $w_j$ of a fairlet $F_j$ is defined as its  cardinality, i.e., number of points in $F_j$. 

Traditional agglomerative clustering approaches merge the two closest clusters, so rely solely on similarity. We extend the merge step by also ensuring that merging does not violate the cardinality constraint w.r.t. the cardinality threshold $q$. 

\begin{theorem}\label{theo: bal}
The balance score of a cluster formed by the union of two or more fairlets, is at least $t$.
\begin{equation*}
    balance(\mathcal{Y})\geq t,~\text{where}~\mathcal{Y}=\cup_{j\le l}F_j  ~\text{and}~ balance(F_j) \geq t
\end{equation*}
\end{theorem}
\begin{proof}
 We use the method of induction to derive the proof. Assume we have a set of fairlets $\mathcal{F} = \{F_1, F_2, \ldots, F_l\}$, ~\text{in which,}~ $balance(F_j)\geq t$, $~j = {1,\ldots,l}$.  We first consider the case for any two fairlets $\{F_1,F_2\}\in\mathcal{F}$. We have \(\displaystyle balance(F_1) = \frac{f_1}{m_1} \geq t\)  and \(\displaystyle balance(F_2) = \frac{f_2}{m_2} \geq t\). We denote by $\mathcal{Y}$ is the union of two fairlets $F_1$ and $F_2$, then
 \begin{equation}\label{eq: bal_f1f2}
     balance(\mathcal{Y})=balance(F_1 \cup F_2)= \frac{f_1+f_2}{m_1+m_2}
 \end{equation}
It holds:
\begin{equation}\label{eq bal_0}
    \begin{aligned}
         \frac{f_1}{m_1}\geq t \phantom{1323331454646431}\\
         \text{or,}~ \frac{f_1}{m_1+m_2}\geq \frac{tm_1}{m_1 +m_2}\phantom{132333131}\\
         \text{Similarly,}~\frac{f_2}{m_1+m_2}\geq \frac{tm_2}{m_1 +m_2}\phantom{132333131}\\
         \implies \frac{f_1}{m_1+m_2} + \frac{f_2}{m_1+m_2} \geq \frac{tm_1}{m_1 +m_2} + \frac{tm_2}{m_1 +m_2}\\
        \implies \frac{f_1+f_2}{m_1+m_2}\geq \frac{t(m_1+m_2)}{m_1+m_2}=t \phantom{13233313156656}
    \end{aligned}
\end{equation}
 Therefore, from Eq.~\ref{eq: bal_f1f2} and Eq.~\ref{eq bal_0} we get,
 \begin{equation}\label{eq: bal_1}
     balance(\mathcal{Y})\geq t
 \end{equation}
Thus, the statement given in Theorem~\ref{theo: bal} is true for any cluster formed by the union of any two fairlets. 
Now we assume that the statement holds true for a cluster formed from $i$ fairlets, i.e, $\mathcal{Y}=\cup_{j\le i}F_j$, where $1<i<l$. Then,
\begin{equation}
    \begin{aligned}
         balance(\mathcal{Y})= \frac{\sum_{j\le i}f_j}{\sum_{j\le i}m_j}\geq t
    \end{aligned}
\end{equation}

Consider another fairlet $F_{i+1}\in \mathcal{F}$ which is not in the formed cluster $\mathcal{Y}$, \(\displaystyle balance(F_{i+1})=\frac{f_{i+1}}{m_{i+1}}\geq t\). Then, by joining $F_{i+1}$ with the cluster $\mathcal{Y}$ we get the new cluster $\mathcal{Y}^{'}$ such that,
\begin{equation}
    \begin{aligned}
         balance(\mathcal{Y}^{'})= \frac{f_{i+1}+ \sum_{j\le i}f_j}{m_{i+1}+\sum_{j\le i}m_j}
    \end{aligned}
\end{equation}
Following the steps in Eq.~\ref{eq bal_0}, we can similarly show that,
\begin{equation}
    \begin{aligned}
         \frac{f_{i+1}+ \sum_{j\le i}f_j}{m_{i+1}+\sum_{j\le i}m_j} \geq t\\
         \implies balance(\mathcal{Y}^{'})\geq t
    \end{aligned}
\end{equation}
Hence, the theorem holds true for cluster formed with $i+1$ fairlets if it is true for $i$ fairlets. Since $i$ is any arbitrary number of fairlets, thus the theorem holds true for all cases.
\end{proof}
Theorem~\ref{theo: bal} shows that for any cluster formed by union of fairlets, the fairness constraint is always preserved. Henceforth, we don't need further interventions w.r.t. fairness.

The pseudocode is shown in Algorithm~\ref{alg:hierarchical}. In each step, the closest pair of clusters is identified (line 4) and a new cluster is created (line 6) only if its capacity does not exceed the capacity threshold $q$. Otherwise, the next closest pair is investigated. The procedure continues until $k$ clusters remain. The remaining clusters are fair and capacitated according to the correponding thresholds $t$ and $q$.

To compute the proximity matrix (line \ref{alg:compute_proximity_1} and line \ref{alg:compute_proximity_2}), we use the distance between centroids of the corresponding clusters. 
The function $capacity(cluster)$ in line \ref{alg:merge_2} returns the size of a cluster.
\begin{algorithm}[htb]
\SetAlgoLined
\KwIn{$\mathcal{F} = \{F_1, F_2, \ldots, F_l\}$: a set of fairlets \newline
$q$: a given maximum capacity of final clusters\newline
$W = \{w_1, w_2, \ldots, w_l\}$: weights of fairlets \newline
$k$: number of clusters
}
\KwOut{A fair-capacitated clustering}
compute the proximity matrix \; \label{alg:compute_proximity_1}
$clusters \gets \mathcal{F} $ //each fairlet $F_j$ is considered as cluster \;
\SetKwRepeat{REPEAT}{repeat}{until}
\REPEAT{$k$ clusters remain}{
    $cluster_1, cluster_2 \gets$ 
    the closest pair of clusters \; \label{alg:merge_1}
    \eIf{$capacity(cluster_1) + capacity(cluster_2) \leq q$}{ \label{alg:merge_2}
        $newcluster \gets$ merge($cluster_1,cluster_2$)\;
        update $clusters$ with $newcluster$\;
        update the proximity matrix \; \label{alg:compute_proximity_2}
        }{
        continue\;
        }
}
\Return{$clusters$}\;
 \caption{Hierarchical fair-capacitated algorithm}
 \label{alg:hierarchical}
\end{algorithm}

\subsection{Fair-capacitated partitioning-based clustering}
\label{subsec:kmedoid_knapsack}

Partitioning-based clustering algorithms, such as $k$-Medoids, can be viewed as a distance minimization problem, in which, we try to minimize the objective function in Eq. \ref{eq:eval}, i.e., minimize the sum of the distance from every $x_j \in X$ to its medoid $s_i$. The vanilla $k$-Medoids does not satisfy a cardinality constraint since the allocating points to clusters step is only based on the distance among them. Now, if we change the goal of this assignment step to find the ``best" data points with a defined capacity for each medoid instead of searching for the most suitable medoid for each point, we can control the cardinality of clusters. We formulate the problem of \emph{assigning points to clusters} subject to a capacity threshold $q$ as a knapsack problem \cite{mathews1896partition}. 

At a given $k$-Medoids assignment step, let $S=\{s_1, s_2,...,s_k\}$ be the cluster centers, i.e., medoids, $\mathcal{C}=\{C_1, C_2,...,C_k\}$ be the resulting clusters. We change the assignments of points to clusters, using knapsack, in order to meet the capacity constraint $q$. In particular, we define a flag variable $y_j = 1$ if $x_j$ is assigned to cluster $C_i$, otherwise $y_j = 0$. Now, if we assign a value $v_j$ to data point $x_j$, which depends on the distance of $x_j$ to $C_i$, with $v_j$ being maximum if $C_i$ is the best cluster for $x_j$, i.e, the distance between $x_j$ and $s_i$ is minimum. We define the value $v_j$ of instance $x_j$ based on an exponential decay distance function:
\begin{equation}
    \label{eq:decay_value}
   v_j = e^{-\frac{1}{\lambda}*d(x_j,s_i)}
\end{equation}
where $d(x_j,s_i)$ is the Euclidean distance between the point $x_j$ and the medoid $s_i$. The higher $\lambda$ is the lower the effect of distance in the value of the points. The point which is closer to the medoid will have a higher value. 

Then, the objective function for the assignment step becomes:
\begin{equation}
\label{eq:value_obj_func}
 \begin{aligned}
      \textrm{maximize}  \sum_{j=1}^{n} v_{j} y_{j}
 \end{aligned}
\end{equation}



Now, given $\mathcal{F} = \{F_1, F_2, \ldots, F_l\}$ and $W = \{w_1, w_2, \ldots, w_l\}$ are the set of fairlets and their corresponding weights respectively; $q$ is the maximum capacity of the final clusters. Our target is to cluster the set of fairlets $\mathcal{F}$ into $k$ clusters centered by $k$ medoids. We apply the formulas in Eq. \ref{eq:decay_value} and Eq.\ref{eq:value_obj_func} on the set of fairlets $\mathcal{F}$, i.e, each fairlet $F_j$ has the same role as $x_j$. Then, the problem of assigning the fairlets to each $medoid$ in the cluster assignment step becomes finding a set of fairlets with the total weights is less than or equal to $q$ and the total value is maximized. In other words, we can formulate the cluster assignment step in the partitioning-based clustering as a \textbf{0-1 knapsack problem}. 
\begin{equation}
\textrm{maximize}  \sum_{j=1}^{l} v_{j} y_{j}
\end{equation}
\begin{equation}
\textrm{subject to} \quad  \sum_{j=1}^{l} w_{j} y_{j} \leq q  \quad \textrm{and} \quad  y_{j} \in\{0,1\}
\end{equation}
In which, $y_j$ is the flag variable for $F_j$, $y_j = 1$ if $F_j$ is assigned to a cluster, otherwise $y_j = 0$ ; $v_j$ is the value of $F_j$ which is computed by the Eq. \ref{eq:decay_value}; $q$ is the desired maximum capacity.

The pseudocode of our $k$-Medoids fair-capacitated algorithm is described in Algorithm 2. 
In which, for each \textit{medoid} we would search for the adequate points (line \ref{alg:assign_medoid}) by using function $knapsack(p, values, w, q)$ (line \ref{alg:knapsack_function})  implemented using dynamic programming, which returns a list of items with a maximum total value and the total weight not exceeding $q$. In the main function, line \ref{alg:k-main_function}, we optimize the clustering cost by replacing $medoids$ with \textit{non-medoid} instances when the clustering cost is decreased. This optimization procedure will stop when there is no improvement in the clustering cost is found (lines \ref{alg:repeat} to \ref{alg:until}).

\setlength{\intextsep}{0pt} 
\begin{algorithm}[htb]
\label{alg:kmedoids_knapsack}
\SetAlgoLined
\KwIn{$\mathcal{F} = \{F_1, F_2, \ldots, F_l\}$: a set of fairlets \newline
$W = \{w_1, w_2, \ldots, w_l\}$: weights of fairlets \newline
$q$: a given maximum capacity of final clusters\newline
$k$: number of clusters
}
\KwOut{A fair-capacitated clustering}

\SetKwFunction{FAssignment}{ClusterAssignment}
\SetKwProg{Fn}{Function}{:}{}
  \Fn{\FAssignment{$medoids$}}{
        $clusters \gets \emptyset$\;
        \For{each medoid $s$ in medoids}{ \label{alg:assign_medoid}
        candidates $\gets$ all fairlets which are not assigned to any cluster \;
        $p$ $\gets$ length(candidates) \;
        $w$ $\gets$ weights(candidates) \;
        \For {each $fairlet_i$ in candidates}{
        $values[i]$ $\gets$ $v(fairlet_i)$ //Eq. \ref{eq:decay_value} \;
        } 
        $clusters[s]$$ \gets $knapsack($p$, $values$, $w$, $q$) \; \label{alg:knapsack_function}
    }
    \KwRet $clusters$\;
  }
\SetKwFunction{FMain}{main} 
\label{alg:k-main_function}
\SetKwProg{Fm}{Function}{:}{}
  \Fm{\FMain{}}{

$medoids \gets$ select $k$ of the $l$ fairlets arbitrarily \;
ClusterAssignment(medoids) \;
$cost_{best}  \gets $ current clustering cost\;
$s_{best}  \gets null$ \;
$o_{best}  \gets null$ \;

\SetKwRepeat{REPEAT}{repeat}{until}
\REPEAT{no improvements can be achieved by any replacement}{
\label{alg:repeat}
\For{each medoid $s$ in medoids}
    {
        \For {each non-medoid $o$ in $\mathcal{F}$}
        {
        consider the swap of $s$ and $o$, compute the current clustering cost\;
        \If{current clustering cost < $cost_{best}$} 
        {
        $s_{best} \gets s$\;
        $o_{best} \gets o$\;
        $cost_{best} \gets current~clustering~cost $\;
        }
        }
    }
update $medoids$ by the swap of $s_{best}$ and $o_{best}$ \;
ClusterAssignment(medoids)
}
}
\label{alg:until}
\Return{$clusters$}\;
\caption{$k$-Medoids fair-capacitated algorithm}
\end{algorithm}


\section{Experiments}
\label{sec:experiment}
In this section, we describe our experiments and the performance of our proposed algorithms on four real educational datasets.  

\subsection{Experimental setup}
\subsubsection{Datasets} 
We evaluate our proposed methods on four public datasets. An overview of datasets is presented in Table \ref{tbl:dataset}.

\begin{table*}[!htb]
\centering
\caption{An overview of the datasets}\label{tbl:dataset}
\begin{tabular}{ccccc}

\hline
Dataset                              & \multicolumn{1}{l}{\#instances} & \multicolumn{1}{l}{\#attributes} & Protected attribute & Balance score \\ \hline
UCI student performance-Mathematics       & 395                            & 33                             & Sex (F: 208, M: 187 ) & 0.899     \\
UCI student performance-Portuguese & 649                            & 33                             & Sex (F: 383; M: 266) & 0.695      \\
PISA test scores                     & 3,404                           & 24                             & Male (1: 1,697; 0: 1,707 )  & 0.994\\
OULAD                                & 4,000                           & 12                             & Gender (F: 2,000; M: 2,000) & 1\\
MOOC                                 & 4,000                           & 21                             & Gender (F: 2,000; M: 2,000) & 1\\ \hline
\end{tabular}
\end{table*}

\textbf{UCI student performance.}
This dataset includes the demographics, grades, social and school-related features of students in secondary education of two Portuguese schools \cite{cortez2008using} in 2005 - 2006 with two distinct subjects: Mathematics and Portuguese. We encode all categorical attributes
by using a one-hot encoding technique. ``Sex" is selected as the protected attribute, i.e., we aim to balance gender in the resulting clusters.

\textbf{PISA test scores.} 
The dataset contains information about the demographics and schools of American students \cite{fleischman2010highlights} taking the exam in 2009 from the Program for International Student Assessment (PISA) distributed by the United States National Center for Education Statistics (NCES). ``Male" which contains two values $\{1,0\}$ is chosen as the protected attribute.

\textbf{Open University Learning Analytics (OULAD).} 
This is the dataset from the OU Analyse project \cite{kuzilek2017open} of Open University, England in 2013 - 2014. Information of students includes their demographics, courses, their interactions with the virtual learning environment, and their final outcome. We aim to balance the ``Gender" attribute in the resulting clusters.
    
\textbf{MOOC.}
The data covers students who enrolled in the 16 edX courses offered by the two institutions (Harvard University and the Massachusetts Institute of Technology) during 2012-2013 \cite{DVN/26147_2014}. The dataset contains aggregated records which represent students' activities and their final grades of the courses. ``Gender" is the protected attribute for fairness in our experiments.

\subsubsection{Baselines} We compare our approaches against well-known clustering methods, including fairness-aware algorithms and a traditional one.

\textbf{$k$-Medoids.} $k$-Medoids clustering \cite{kaufman1990partitioning}
is a traditional partitioning technique of clustering that divides the dataset into $k$ clusters and minimizes the clustering cost. $k$-Medoids uses the actual instances as centers.

\textbf{Vanilla fairlet.} This is the approach proposed by Chierichetti et al. \cite{chierichetti2017fair}. The first phase computes a vanilla fairlet decomposition that ensures fair clusters, but it might not give the optimal cost value. A vanilla $k$-Center clustering algorithm \cite{gonzalez1985clustering} is applied to cluster those fairlets into $k$ clusters in the second step.

\textbf{MCF fairlet.} In this version \cite{chierichetti2017fair}, the fairlet decomposition is transformed into a \textit{minimum cost flow} (MCF) problem, by which an optimized version of fairlet decomposition in terms of cost value is computed. Like the vanilla version, a $k$-Center method is used to cluster fairlets in the second phase.

In our experiments, both resulting fairlets generated by vanilla fairlet and MCF fairlet methods are used for our proposed fair-capacitated clustering algorithms. Therefore, we have two versions each proposed methods, namely \textbf{Vanilla fairlet hierachical fair-capaciated} and \textbf{MCF fairlet hierachical fair-capacitated} (for the hierarchical approach), \textbf{Vanilla fairlet k-Medoids fair-capacitated} and \textbf{MCF fairlet k-Medoids fair-capacitated} (for the partitioning approach). Section \ref{subsec:experimental_results} presents the experimental results of these methods.
\subsubsection{Measurement} We report our experimental results on clustering cost, balance score, and capacity. 
The \textit{clustering cost} is used for evaluating the quality of clustering, which is measured by the formula given in Eq. \ref{eq:eval}. The fairness of clustering is measured by the \textit{balance} score in Eq. \ref{eq:clustering_balance}.

\subsubsection{Parameter selection}
\label{subsubsec:Parameter}
Regarding fairness, a minimum threshold of balance $t$ is set to 0.5 for all datasets in our experiments. It means that the proportion of the minority group (e.g.,: female) is at least 50\% in the resulting clusters.
Regarding the $\lambda$ factor in Eq. \ref{eq:decay_value}, a value $\lambda = 0.3$ is chosen for our experiments from a range of [0.1, 1.0] via grid-search. We evaluate the clustering cost and balance score on a small dataset, i.e., UCI student performance dataset - Mathematics subject w.r.t $\lambda$. 

Theoretically, the \textbf{ideal capacity} of clusters is \(\displaystyle  \Big \lceil \frac{|X|}{k}  \Big \rceil\) where $|X|$ is the population of dataset $X$, $k$ is the number of desired clusters. However, in many cases, the clustering models cannot satisfy this constraint, especially the hierarchical clustering model. Hence, we set the formula in Eq. \ref{eq:capacity} to compute the \textbf{maximum capacity} $q$ of clusters; $\varepsilon$ is a parameter chosen in experiments for each fair-capacitated clustering approach.
\begin{equation}
\label{eq:capacity}
    q = \Big \lceil \frac{|X|*\varepsilon}{k} \Big \rceil 
\end{equation}
In our experiments, to find the appropriate value of $\varepsilon$, we set a range of [1.0, 1.3] to ensure all the generated clusters have members. We evaluate the cardinality of resulting clusters on the UCI student performance dataset - Mathematics subject. Based on this, $\varepsilon$ is set to 1.01 and 1.2, for $k$-Medoids fair-capacitated and hierarchical fair-capacitated methods, respectively.

\subsection{Experimental results}
\label{subsec:experimental_results}

\subsubsection{UCI student performance - Mathematics}
In Figure 
\ref{fig:student_mat}-a, the clustering cost of all methods is worse compared to that of the vanilla $k$-Medoids clustering. This is expected as these methods have to satisfy constraints on fairness or/and cardinality. However, both of our approaches outperform the vanilla fairlet and MCF fairlet methods. In which, \textit{MCF fairlet hierarchical fair-capacitated} shows the best performance due to the optimization in the merging step.
Regarding fairness, as shown in Figure \ref{fig:student_mat}-b, the minimum threshold of balance $t$ is visualized as a dashed line while the actual balance from the dataset is plotted as a dotted line. All of our methods are comparative to the competitors in almost cases. Interestingly, our \textit{vanilla fairlet $k$-Medoids fair-capacitated} method outperforms the competitive methods when $k$ is less than 10. In terms of cardinality, as presented in Figure \ref{fig:student_mat}-c, the maximum capacity thresholds $q$ are indicated by the figure's dashed and dotted lines. Our capacitated variants are superior (lower dispersion as shown by the interquartile ranges). We have to thicken the boxplots of our proposed methods since in many cases the dispersion in the size of the resulting clusters is too small.
MCF fairlet shows the worst performance in terms of cardinality, followed by Vanilla fairlet and vanilla $k$-Medoids.  

\begin{figure*} [!htb]
     \centering
     \begin{subfigure}[b]{1.0\textwidth}
         \centering
         \label{fig:plot_student_mat}
         \includegraphics[width=\textwidth]{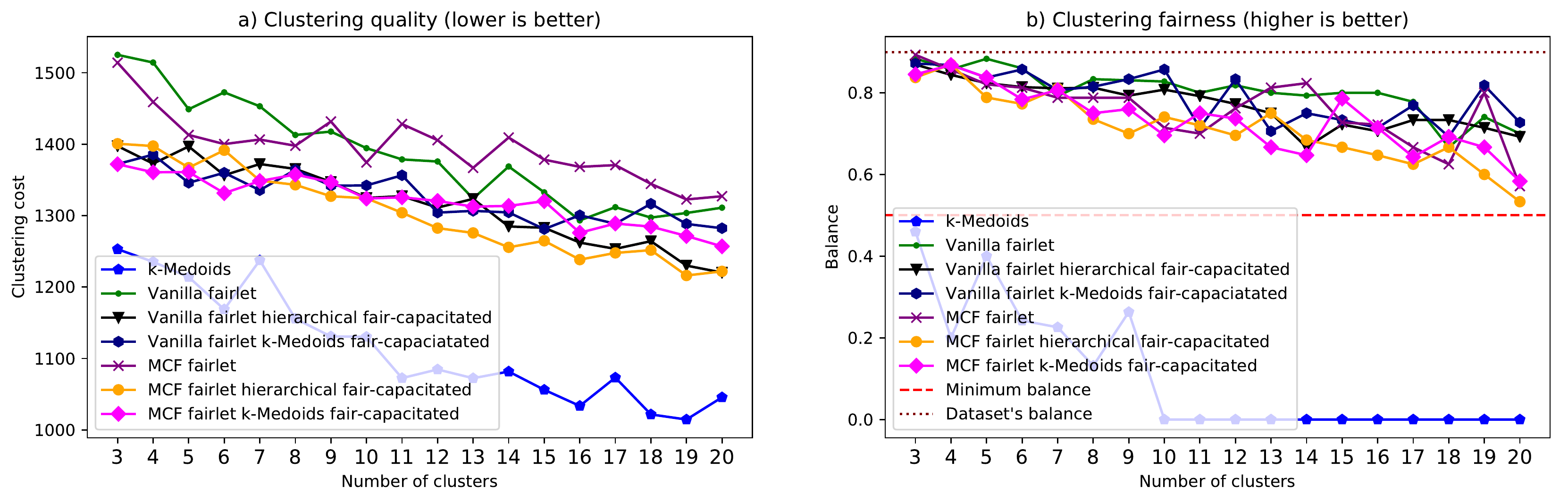}
         
         \vspace{-10pt}
     \end{subfigure}
     \hfill
     \begin{subfigure}[b]{1.0\textwidth}
         \centering
         \includegraphics[width=\textwidth]{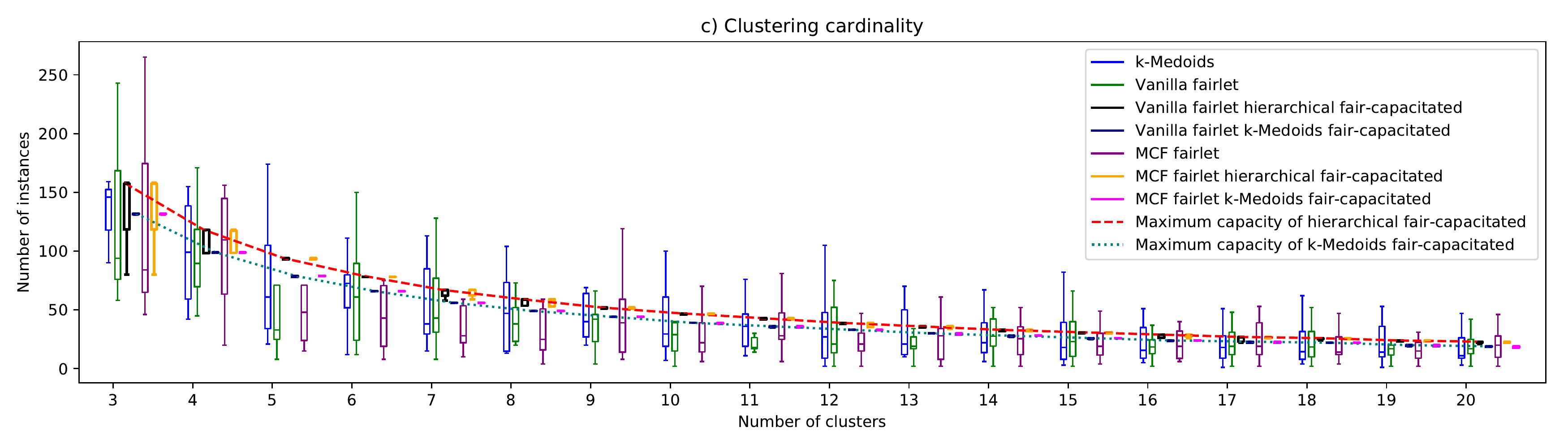}
         \label{fig:boxplot_student_mat}
         
     \end{subfigure}
     \vspace{-25pt}
     \caption{Performance of different methods on UCI student performance dataset - Mathematics subject}
     
     \label{fig:student_mat}
\end{figure*}

\subsubsection{UCI student performance - Portuguese}
When $k$ is less than 4, as shown in Figure \ref{fig:student_por}-a, the clustering quality of our models can be close to that of the vanilla $k$-Medoids method. However, the clustering cost is fluctuated thereafter due to the effort to maintain the fairness and cardinality of methods. Our \textit{vanilla fairlet hierarchical fair-capacitated} outperforms other competitors in most cases. Vanilla fairlet and MCF fairlet show the worst clustering cost as an effect of the $k$-Center method. Figure \ref{fig:student_por}-b depicts the clustering fairness. As we can observe, in terms of fairness, \textit{vanilla fairlet hierarchical fair-capacitated} has the best performance when $k$ is less than 10. Contrary to that, by selecting each point for each cluster in the cluster assignment step, the \textit{$k$-Medoids fair-capacitated} method can maintain well the fairness in many cases. Regarding the cardinality, as illustrated in Figure \ref{fig:student_por}-c, our approaches outperform the competitors when they can keep the number of instances for each cluster under the specified thresholds. 

\begin{figure*} [!htb]
     \centering
     \begin{subfigure}[b]{1.0\textwidth}
         \centering
         
         \includegraphics[width=\textwidth]{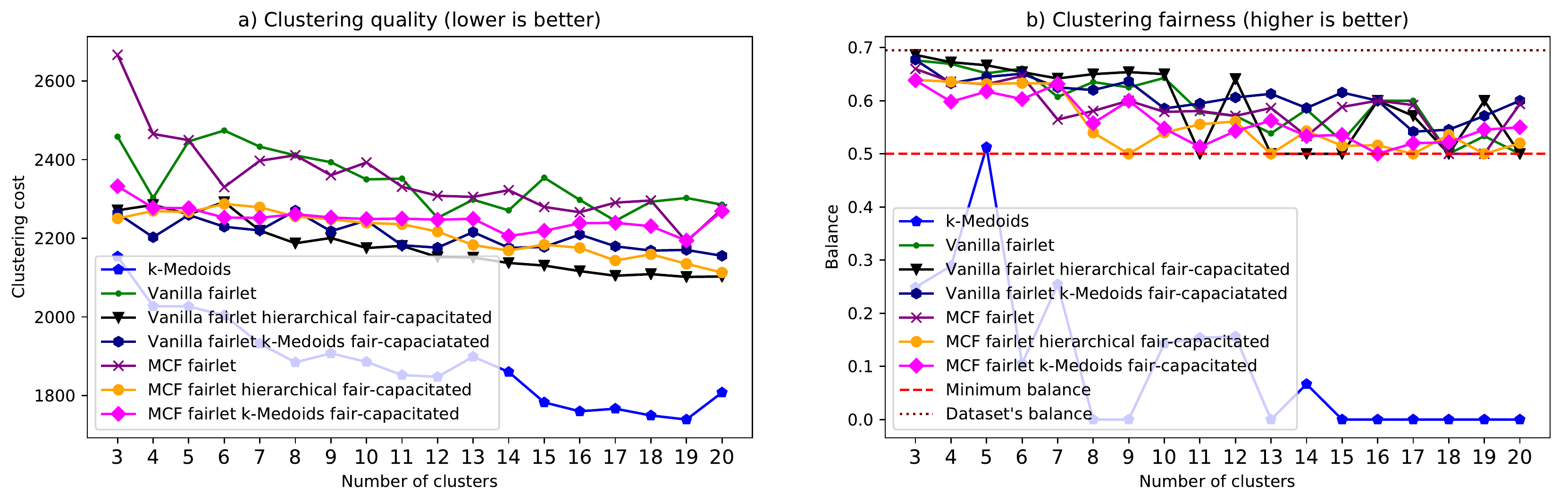}
         \label{fig:plot_student_por}
         \vspace{-10pt}
     \end{subfigure}
     \hfill
     \begin{subfigure}[b]{1.0\textwidth}
         \centering
         \includegraphics[width=\textwidth]{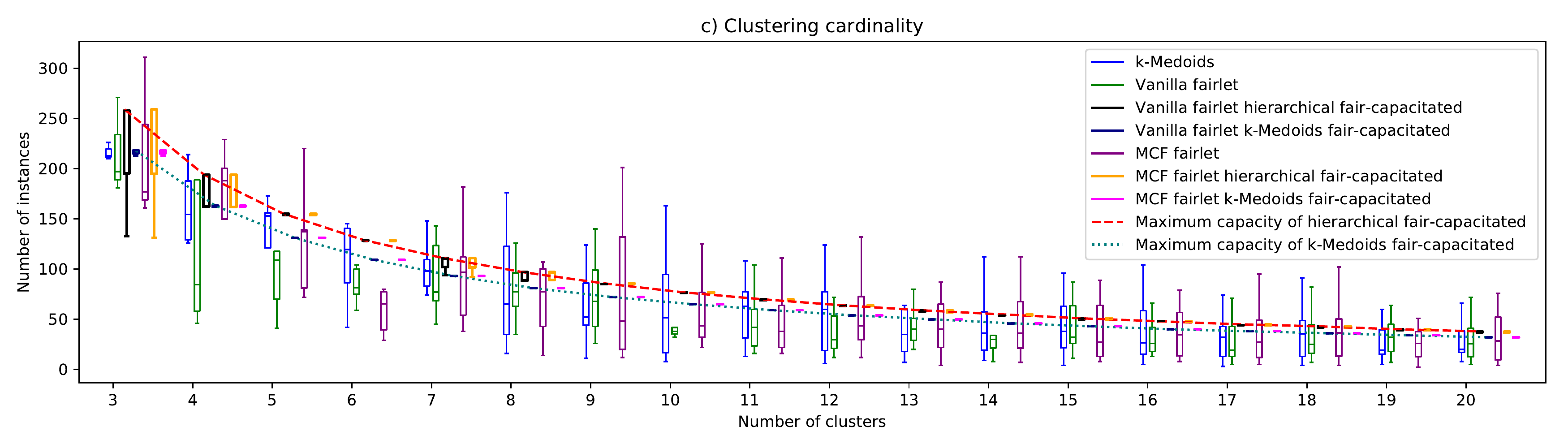}
         \label{fig:boxplot_student_por}
         
     \end{subfigure}
     \vspace{-25pt}
     \caption{Performance of different methods UCI student performance dataset - Potuguese subject}
     
     \label{fig:student_por}
\end{figure*}

\subsubsection{PISA test scores}
Although the clustering cost increases in most methods, as presented in Figure \ref{fig:pisa}-a, our approaches outperform the competitors vanilla fairlet and MCF fairlet. The hierarchical approach shows the best performance compare to other methods which are concerned with equity and capacity. Interestingly, our proposed methods outperform the competitors when they can preserve very well the balance score for all number of clusters in terms of fairness (Figure \ref{fig:pisa}-b). This is explained by fairness in the fairlets that are used as the input for our clustering method. It is easy to observe in Figure \ref{fig:pisa}-c that our proposed methods strictly follow the maximum capacities of clusters regarding the cardinality. MCF fairlet is the worst model, followed by $k$-Medoids and vanilla fairlet and MCF fairlet.

\begin{figure*} [!htb]
     \centering
     \begin{subfigure}[b]{1.0\textwidth}
         \centering
         
         \includegraphics[width=\textwidth]{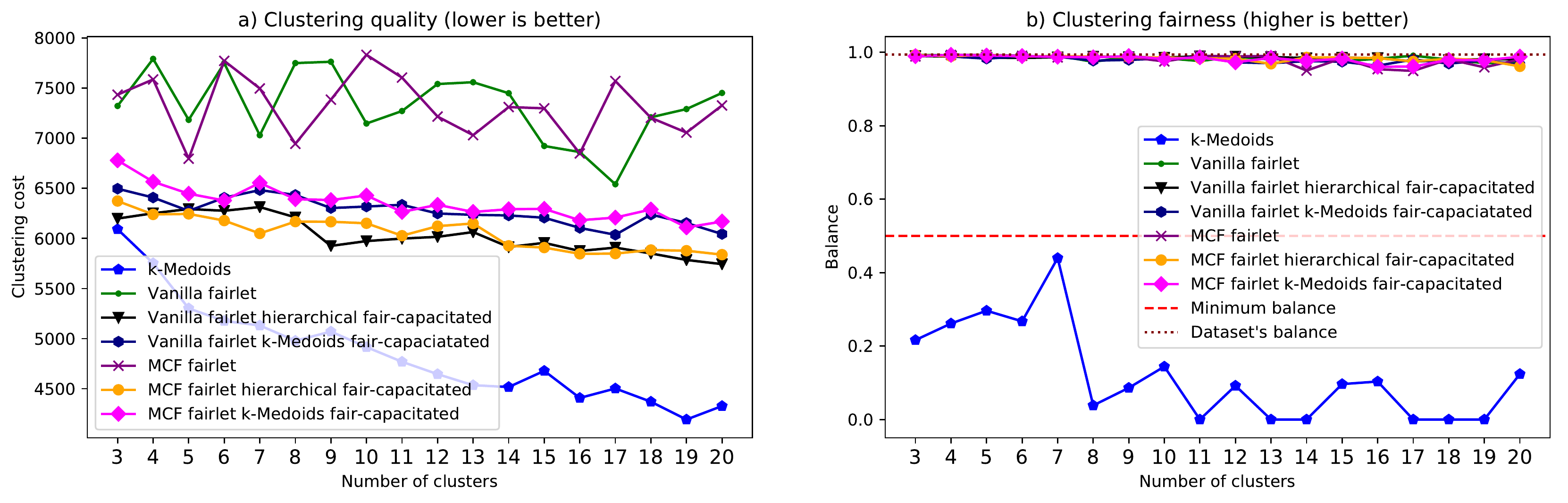}
         \label{fig:plot_pisa}
         \vspace{-10pt}
     \end{subfigure}
     \hfill
     \begin{subfigure}[b]{1.0\textwidth}
         \centering
         \includegraphics[width=\textwidth]{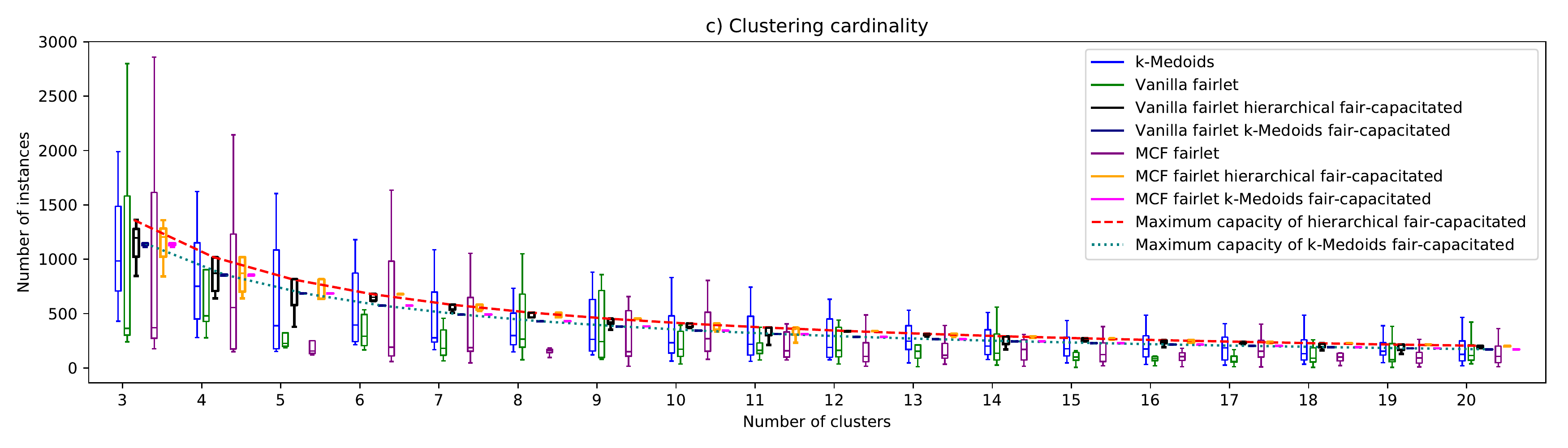}
         \label{fig:boxplot_pisa}
         
     \end{subfigure}
     \vspace{-25pt}
     \caption{Performance of different methods on PISA Test Score dataset}
     
     \label{fig:pisa}
\end{figure*}

\subsubsection{OULAD}
Our \textit{MCF fairlet $k$-Medoids fair-capacitated} approach outperforms other methods in terms of clustering cost, although there is an increase compared to the vanilla $k$-Medoids algorithm, as we can see in Figure \ref{fig:oulad}-a. Concerning fairness, in Figure \ref{fig:oulad}-b, $k$-Medoids is the weakest method while others can achieve the highest balance. The balance of \textit{Gender} feature in the dataset is the main reason for this result. All fairlets are fully fair; this is a prerequisite for our methods of being able to maintain the perfect balance. Regarding cardinality, our approaches demonstrate their strength in ensuring the capacity of clusters (Figure \ref{fig:oulad}-c). The difference in the size of the clusters generated by our methods is tiny. This is in stark contrast to the trend of competitors.

\begin{figure*} [!htb]
     \centering
     \begin{subfigure}[b]{1.0\textwidth}
     
         \centering
         \includegraphics[width=\textwidth]{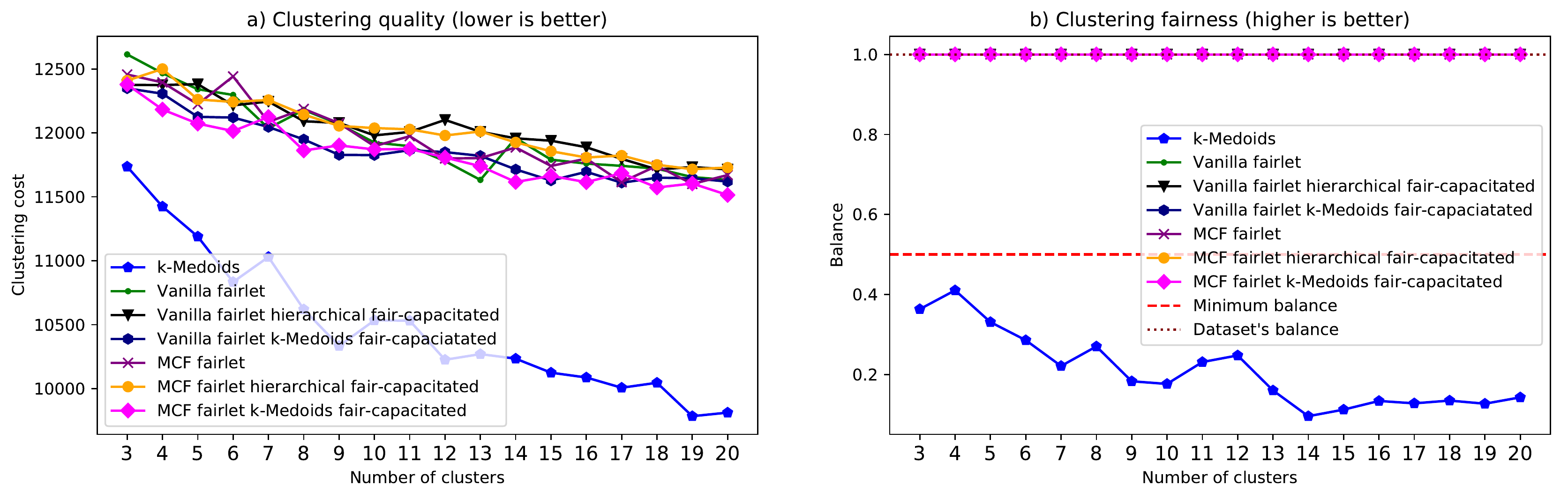}
         \label{fig:plot_oulad}
          \vspace{-10pt}
     \end{subfigure}
     \hfill
     \begin{subfigure}[b]{1.0\textwidth}
         \centering
         \includegraphics[width=\textwidth]{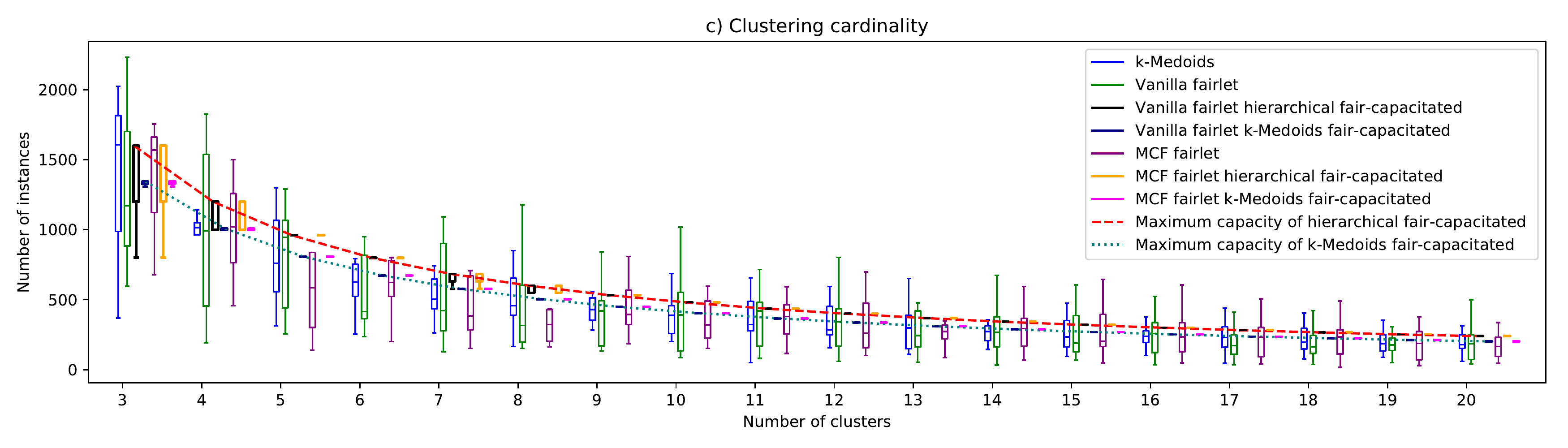}
         \label{fig:boxplot_oulad}
         
     \end{subfigure}
     \vspace{-25pt}
     \caption{Performance of different methods on OULAD dataset}
     
     \label{fig:oulad}
\end{figure*}

\subsubsection{MOOC}
The results of clustering quality are described in Figure \ref{fig:mooc}-a. Although an increase in the clustering cost is the main trend, our methods outperform the vanilla fairlet and MCF fairlets methods. Regarding clustering fairness, as depicted in Figure \ref{fig:mooc}-b, our approaches can maintain the perfect balance for all experiments. This is the result of the actual balance in the dataset and the fairlets. The emphasis is our methods can divide all the experimented instances into capacitated clusters, as presented in Figure \ref{fig:mooc}-c, which proves their superiority in presenting the results over the competitors regarding the cardinality of clusters.
\begin{figure*} [!htb]
     \centering
     \begin{subfigure}[b]{1.0\textwidth}
         \centering
         
         \includegraphics[width=\textwidth]{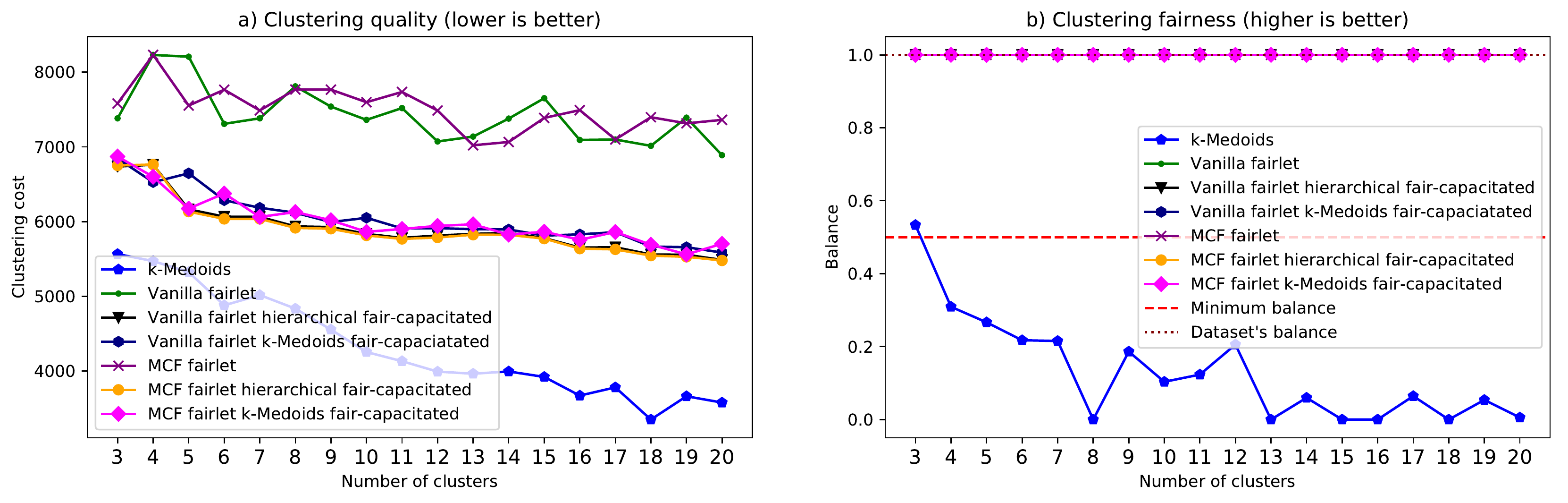}
         \label{fig:plot_mooc}
         \vspace{-10pt}
     \end{subfigure}
     \hfill
     \begin{subfigure}[b]{1.0\textwidth}
         \centering
         \includegraphics[width=\textwidth]{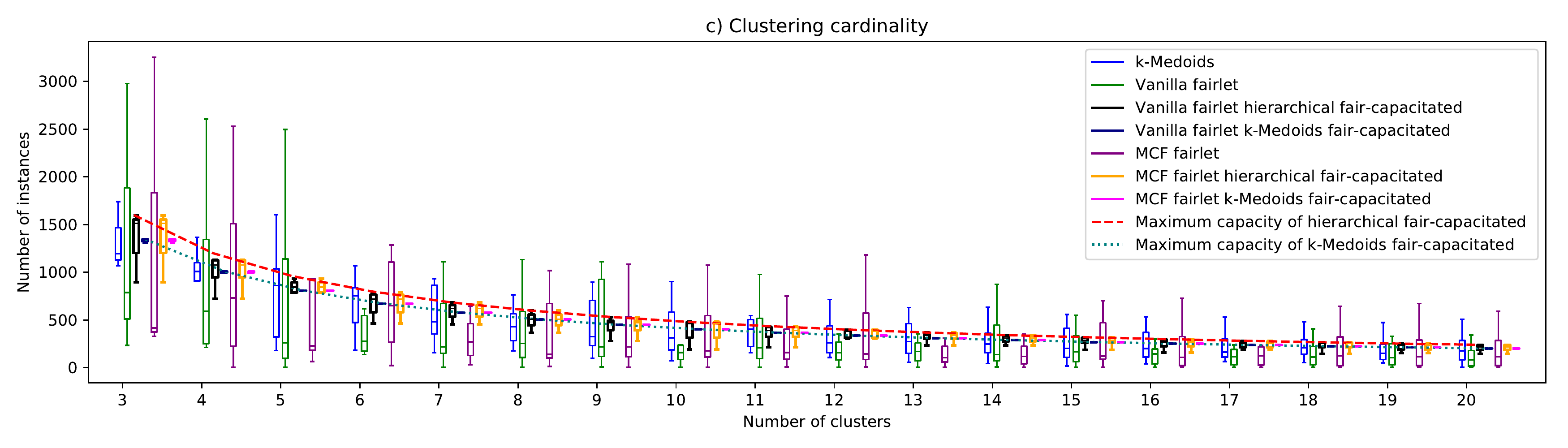}
         \label{fig:boxplot_mooc}
     \end{subfigure}
     \vspace{-25pt}
     \caption{Performance of different methods on MOOC dataset}
     \label{fig:mooc}
\end{figure*}

\subsubsection{Summary of the results}

In general, fairness is well maintained in all of our experiments. When the data is fair, in case of OULAD and MOOC datasets, our methods achieve perfect fairness. In terms of cardinality, our methods are able to maintain the cardinality of resulting clusters within the maximum capacity threshold, which is significantly superior to competitive methods.
The fair-capacitated partitioning based method is better than the hierarchical approach since it can determine the capacity threshold closest to the \textit{ideal cardinality} mentioned in Section \ref{subsubsec:Parameter}. Regarding the clustering cost, the hierarchical approach has an advantage over other methods by outperforming its competitors in most experiments.

\section{Conclusion and Outlook}
\label{sec:conclusion}

In this work, we introduced the fair-capacitated clustering problem that extends traditional clustering, solely focusing on similarity, by also aiming at a balanced cardinality among the clusters and a fair-representation of instances in each cluster according to some protected attribute like gender or race. Our solutions work on the fairlets derived from the original instances: the hierarchical-based approach takes into account the cardinality requirement during the merging step, whereas the partitioning-based approach takes into account the cardinality of the final clusters during the assignment step which is formulated as a knapsack problem. Our experiments show that our methods are effective in terms of fairness and cardinality while maintaining clustering quality. Apart from the educational field, the fair-capacitated clustering problem can contribute to other applications such as clustering of customers in marketing studies, vehicle routing and communication network design.
An immediate future direction is to improve the clustering quality by optimizing the cluster assignment phase of the partitioning-based approaches. Moreover, we plan to extend our work for multiple protected attributes.

%
\bibliographystyle{abbrv}
\bibliography{mybibliography} 
%
%

\end{document}